\PassOptionsToPackage{numbers,compress}{natbib}
\documentclass{article}

\usepackage[preprint]{neurips_2025}
\pdfoutput=1

\usepackage{amsmath,amssymb,amsthm}
\usepackage[pdftex]{graphicx}
\graphicspath{{figures/}}
\usepackage[utf8]{inputenc}
\usepackage[T1]{fontenc}
\usepackage{booktabs}
\usepackage{amsfonts}
\usepackage{microtype}
\usepackage[dvipsnames]{xcolor}
\usepackage{hyperref}
\usepackage{enumitem}

\newtheorem{theorem}{Theorem}
\newtheorem{proposition}[theorem]{Proposition}

\newtheorem{corollary}[theorem]{Corollary}
\newtheorem{remark}{Remark}
\newtheorem{definition}{Definition}

\DeclareMathOperator{\diag}{diag}
\DeclareMathOperator{\vecop}{vec}

\title{Spectral Asymptotics of Neural Network Loss Landscapes:\\
An Exact Decomposition of the Curvature Exponent}

\author{%
\textbf{Anherutowa Calvo\textsuperscript{1}}\\[4pt]
\textsuperscript{1}9D Labs \quad
\texttt{a@9dlabs.xyz}\\[2pt]
}

\begin{document}
\maketitle

\begin{abstract}
The curvature exponent $\alpha$ in $h_k \propto \sigma_k^\alpha$---governing how Hessian eigenvalues scale with gradient singular values---varies systematically across layer types ($\alpha \approx 2$ for convolutions, $\approx 1$ for transformer attention, $< 1$ for MLP up-projections). Why? We prove the \textbf{Spectral Alignment Decomposition}: $\alpha = 2 + d\log\Phi_k / d\log\sigma_k$, where $\Phi_k$ measures alignment between Kronecker factor eigenbases and gradient singular directions. This reduces ``why does $\alpha$ vary?'' to a geometric question we answer for LayerNorm, residual connections, and softmax heads. The decomposition implies a \textbf{spectral transfer identity} $s = \alpha\gamma$ linking curvature exponent, effective gradient rank-decay $\gamma$, and Hessian decay exponent $s$~\cite{tang2025hessian}. The identity is algebraic; its empirical content is that $\alpha$ and $\gamma$, fit on \emph{independent} data (HVPs vs.\ SVD), recover $s$ to $\approx$2\% median error across 93 layers, five architectures, and three datasets---with no free parameters. A zeta-function bound on participation ratio shows curvature concentrates onto effectively one direction per layer. As a proof of concept, we derive the architecture-adaptive preconditioner $T(\sigma;\alpha)$ and show that \textbf{Spectral Newton}---implementing $T$ in the gradient singular basis---outperforms AdamW on vision benchmarks where $\alpha \approx 2$.
\end{abstract}

\section{Introduction}

The eigenvalue spectrum of the neural network loss Hessian encodes how optimization navigates parameter space. Recent work has established \emph{what} the spectrum looks like: Hessian eigenvalues decay as power laws with exponent $s$~\cite{tang2025hessian,wu2020dissecting}, weight matrices exhibit heavy-tailed spectral densities correlated with generalization~\cite{martin2021implicit}, and Kronecker-factored approximations capture much of the Hessian structure~\cite{martens2015kfac,grosse2016kfc}. These results describe \emph{different projections} of the same underlying object, but the mechanistic question---\emph{why} $s$ takes the values it does, and how it connects to gradient structure---remains unanswered.

We study curvature \emph{along gradient singular directions}. For layer $\ell$ with gradient $G_\ell = U\Sigma V^\top$, let $h_k$ denote the exact Hessian eigenvalue along direction $u_k v_k^\top$, measured via Hessian-vector products (HVPs). Empirically, $h_k$ follows a power law in the gradient singular value $\sigma_k$:
\begin{equation}\label{eq:power-law}
    h_k = c \cdot \sigma_k^\alpha
\end{equation}
The exponent $\alpha$ is \emph{not universal}: it depends on layer type, architecture, and task (Section~\ref{sec:empirical}). Understanding $\alpha$ is the key to predicting which preconditioner a layer requires.

Our central contribution is the \textbf{Spectral Alignment Decomposition} (Theorem~\ref{thm:spectral-alignment}): an exact decomposition of $h_k$ into Kronecker factor eigenvalues and measurable alignment ratios, yielding an identity for $\alpha$ in terms of log-log slopes of these quantities. We then:
\begin{enumerate}[leftmargin=*,nosep]
    \item Derive mechanism-specific predictions for LayerNorm ($\alpha \approx 1$), residual connections, and softmax heads ($\alpha > 2$), validated by controlled ablations.
    \item Establish the \textbf{spectral transfer identity} $s = \alpha\gamma$: an algebraic link between $\alpha$, an effective rank-decay exponent $\gamma$ fit on gradient singular values, and Hessian decay $s$~\cite{tang2025hessian}, validated on 89 layers across three datasets (Section~\ref{sec:spectral}).
    \item Bound the curvature participation ratio via the Riemann zeta function, showing learning is effectively one-dimensional per layer.
    \item Prescribe architecture-adaptive spectral transfer functions $T(\sigma;\alpha)$ and validate \textbf{Spectral Newton} as their optimizer instantiation (Section~\ref{sec:optimizer-validation}).
\end{enumerate}

This is a spectral asymptotics result in the spirit of Weyl's law: the decay exponent encodes geometric information (alignment structure) of the loss landscape, not merely a fitted phenomenological parameter.

\section{Setup and Measurement}\label{sec:setup}

For layer $\ell$ with weight $W_\ell \in \mathbb{R}^{m \times n}$, gradient $G_\ell = \nabla_{W_\ell}\mathcal{L}$, and mini-batch activations $A \in \mathbb{R}^{B \times n}$, errors $\delta \in \mathbb{R}^{B \times m}$, the gradient factors as $G_\ell = \frac{1}{B}\delta^\top A$. Under the Gauss--Newton (GN) approximation, the per-layer Hessian satisfies $\mathcal{H}_{\mathrm{GN}}^{(\ell)} \approx C_\delta \otimes C_A$ with $C_\delta = \frac{1}{B}\delta^\top\delta$ and $C_A = \frac{1}{B}A^\top A$.

\paragraph{Measurement protocol.} We train models on CIFAR-10 and Tiny-ImageNet-200 to convergence, and also evaluate pretrained ImageNet-1K weights (IMAGENET1K\_V1). For each, we fix a batch of $B=2048$--$4096$ samples, compute top-$k$ gradient singular directions via SVD, and measure $h_k = v_k^\top \mathcal{H} v_k$ via double backpropagation (exact HVP). We fit $\alpha$ by log-log regression on $(\sigma_k, h_k)$. All results use exact HVPs; finite-difference Hessians yield $R^2 \approx 0.17$ with 50\% spurious negative curvatures and are excluded.

\paragraph{Exponents $\alpha$ and $\gamma$.} We fit $\alpha$ from exact HVPs on $(\sigma_k, h_k)$. We define $\gamma$ as the negated log-log slope of rank-ordered $\sigma_k$ on top-$k$ gradient singular values (typically $k{=}20$; $k{=}100$ for extended spectra in Appendix~\ref{app:gamma-profile}). This is an \emph{effective} decay exponent over a finite rank window, not a claim that $\sigma_k$ follows a global power law.

\section{The Curvature Exponent: An Empirical Landscape}\label{sec:empirical}

Figure~\ref{fig:power-law} visualizes the core empirical law $h_k \propto \sigma_k^\alpha$ on representative layers; Figure~\ref{fig:alpha-depth} shows that $\alpha$ is not random noise but tracks architecture and depth. Table~\ref{tab:arch-sweep} summarizes means across models: convolutional layers universally exhibit $\alpha \approx 2$; transformer layers cluster near $\alpha \approx 1$; output heads can exceed $\alpha = 4$.

\begin{figure}[t]
\centering
\includegraphics[width=\linewidth]{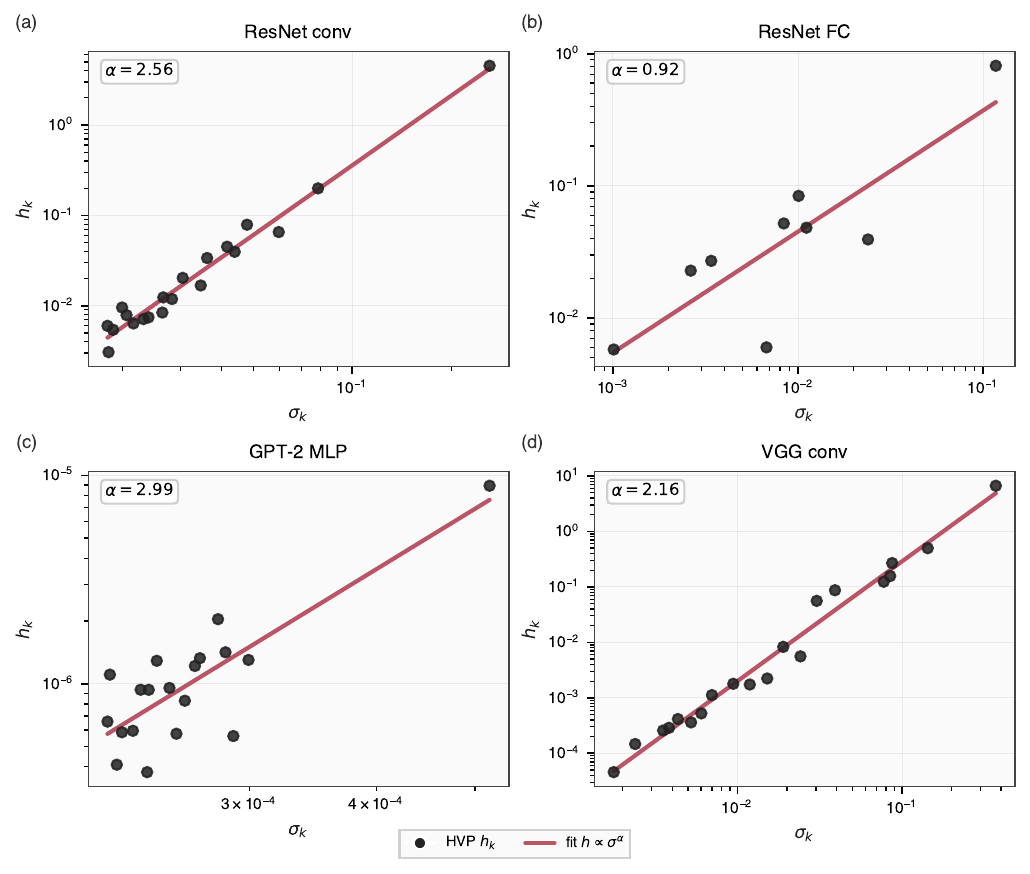}
\caption{Exact HVP curvature $h_k$ vs.\ gradient singular value $\sigma_k$ (log--log). \textbf{(a--d)} Representative layers: convolutions follow $\alpha \approx 2$; FC and transformer MLP layers deviate. Fit exponents shown in-panel; dashed grid omitted on log axes for clarity.}
\label{fig:power-law}
\end{figure}

\begin{figure}[t]
\centering
\includegraphics[width=0.78\linewidth]{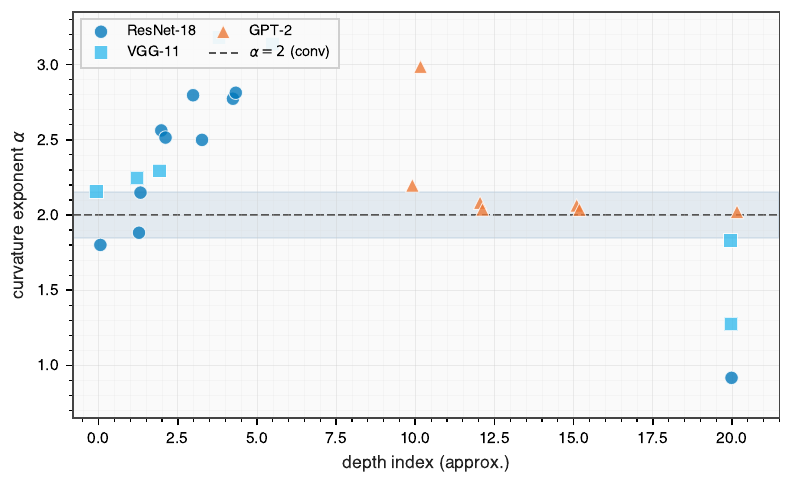}
\caption{Per-layer curvature exponent $\alpha$ vs.\ depth index across ResNet-18, VGG-11, and GPT-2 (CIFAR-10). Shaded band: $\alpha \in [1.85, 2.15]$. Interior convolutions cluster near $\alpha=2$; transformers and boundary layers deviate.}
\label{fig:alpha-depth}
\end{figure}

\begin{table}[t]
\centering
\caption{Mean curvature exponent $\alpha$ by architecture and layer type. CIFAR-10 unless noted; $^\dagger$Tiny-ImageNet-200; $^\ddagger$ImageNet-1K pretrained (IMAGENET1K\_V1).}
\label{tab:arch-sweep}
\small
\begin{tabular}{llcc}
\toprule
\textbf{Architecture} & \textbf{Layer type} & Mean $\alpha$ & Std \\
\midrule
ResNet-18 / VGG-11 (CIFAR-10) & Conv & 1.97--2.17 & 0.08--0.12 \\
ResNet-50$^\dagger$ (Tiny-ImageNet) & Conv ($n=49$) & 1.94 & 0.15 \\
ResNet-50$^\ddagger$ (ImageNet-1K) & Conv ($n=14$) & 2.10 & 0.14 \\
Pure MLP & Dense & 2.18 & 0.10 \\
\midrule
ResNet-50$^\dagger$ (Tiny-ImageNet) & FC (200 classes) & 0.90 & --- \\
ResNet-50$^\ddagger$ (ImageNet-1K) & FC (1000 classes) & 2.83 & --- \\
VGG-11 / ResNet-50 (CIFAR-10) & FC classifier & 1.23--1.65 & 0.07 \\
GPT-2 (6-layer) & Attention & 1.30 & 0.15 \\
GPT-2 & MLP & 1.00 & 0.18 \\
GPT-2 & Output head & 4.53 & --- \\
\bottomrule
\end{tabular}
\end{table}

\begin{table}[t]
\centering
\caption{Transformer layer detail (mini GPT-2, 5 epochs, synthetic LM).}
\label{tab:transformer}
\small
\begin{tabular}{lccc}
\toprule
\textbf{Layer type} & Mean $\alpha$ & Range & $n$ \\
\midrule
Attention QKV & 1.43 & $[1.21, 1.63]$ & 6 \\
Attention projection & 1.17 & $[1.07, 1.34]$ & 6 \\
MLP up-projection & 0.84 & $[0.69, 1.04]$ & 6 \\
MLP down-projection & 1.15 & $[1.13, 1.19]$ & 6 \\
\bottomrule
\end{tabular}
\end{table}

\noindent Conv layers maintain $\alpha \approx 2$ throughout training (emerging from $\sim 1.5$ at initialization; Appendix~\ref{app:training-dynamics}). This holds at scale: ResNet-50 on Tiny-ImageNet (200 classes, 64$\times$64) yields median $\alpha = 1.93$ across 49 conv layers with $R^2 \geq 0.88$. \textbf{On ImageNet-1K} with pretrained weights (IMAGENET1K\_V1), 14 conv layers have median $\alpha = 2.15$ with $R^2 \geq 0.97$ and spectral transfer error 1.6\% (Table~\ref{tab:arch-sweep}). The FC head exhibits opposite trends at different scales: $\alpha = 0.90$ with 200 classes (Tiny-ImageNet) vs.\ $\alpha = 2.83$ with 1000 classes (ImageNet-1K). Both are consistent with the theory: 200 classes produce sparse anisotropic $C_\delta$, pulling $\alpha$ below 2; 1000 classes give $C_\delta$ richer structure that strengthens alignment and pushes $\alpha$ above 2 (Section~\ref{sec:mechanisms}). Full spectrum analysis shows $\alpha$ is stable across quartiles of the singular value index $k$ (Appendix~\ref{app:full-spectrum}).

For convolutional layers, $h_k \propto \sigma_k^2$ holds with $R^2 = 0.98$ (median across 21 ResNet-18 layers; full table in Appendix~\ref{app:hvp-table}).

\section{Related Work}\label{sec:related}

\paragraph{Hessian structure and closed-form spectra.}
Tang et al.~\cite{tang2025hessian} \emph{discover} power-law Hessian eigenvalue decay $h_k \propto k^{-s}$ across CNNs and LLMs and use it to predict generalization. Their contribution is the observation and its statistical characterization; they do not explain \emph{why} $s$ takes the values it does or how $s$ relates to gradient structure. Our work \emph{decomposes} $s$ via the spectral transfer identity $s = \alpha\gamma$: the Hessian decay exponent is the product of a curvature--gradient alignment exponent $\alpha$ (which we trace to Kronecker factor eigenbases, Theorem~\ref{thm:spectral-alignment}) and an effective gradient rank-decay exponent $\gamma$. In short, Tang et al.\ measure $s$; we explain $s$. Wu et al.~\cite{wu2020dissecting} identify shared Hessian structure across architectures. Recent work derives \emph{closed-form} Hessian spectra for shallow networks~\cite{singh2026cracking}, complementing our layer-wise \emph{structural} decomposition for deep trained models.

\paragraph{Kronecker factorization and layer-wise preconditioning.}
K-FAC~\cite{martens2015kfac,grosse2016kfc} and Shampoo~\cite{gupta2018shampoo} approximate curvature with Kronecker factors $C_\delta \otimes C_A$. Zhang et al.~\cite{zhang2025concurrence} prove that layer-wise Kronecker preconditioning is \emph{provably necessary} for feature learning in certain regimes---our Spectral Alignment Decomposition gives a per-layer spectral explanation (alignment ratios $\rho_k$, $\cos^2\theta_k$) for when the approximation is accurate. ESO~\cite{chen2026eso} builds efficient spectral preconditioners via truncated Shampoo; our $s=\alpha\gamma$ identity predicts which layers tolerate low-rank spectral truncation (steep $s$, concentrated curvature) vs.\ which require richer structure.

\paragraph{Heavy-tailed self-regularization.}
Martin--Mahoney~\cite{martin2021implicit} relate weight matrix tail exponents to generalization. Corollary~\ref{cor:htsr} connects their $\alpha_{\mathrm{weight}}$ to our $(\alpha, \gamma, s)$ triad.

\paragraph{Spectral methods in optimization.}
Muon~\cite{jordan2024muon} flattens gradient singular values; SAM~\cite{foret2021sam} explicitly minimizes sharpness. Our decomposition shows the architecture-appropriate spectral weighting is $T(\sigma;\alpha) = \sigma/(\sigma^\alpha + d)$---a consequence of alignment geometry. \textbf{Spectral Newton} implements this transfer in the gradient singular basis; Section~\ref{sec:optimizer-validation} validates it on vision benchmarks where $\alpha \approx 2$.

\section{The Spectral Alignment Decomposition}\label{sec:decomposition}

\begin{definition}[Spectral alignment ratios]\label{def:alignment}
Let $G = U\Sigma V^\top$ be the gradient SVD, and $C_\delta = Q_\delta \Lambda_\delta Q_\delta^\top$, $C_A = Q_A \Lambda_A Q_A^\top$ eigendecompositions with eigenvalues in decreasing order. Define:
\[
    \rho_k^{(\delta)} = \frac{u_k^\top C_\delta u_k}{[\Lambda_\delta]_{kk}}, \qquad
    \rho_k^{(A)} = \frac{v_k^\top C_A v_k}{[\Lambda_A]_{kk}}
\]
\end{definition}

\begin{theorem}[Spectral Alignment Decomposition]\label{thm:spectral-alignment}
Under $\mathcal{H}_{\mathrm{GN}}^{(\ell)} \approx C_\delta \otimes C_A$:
\begin{enumerate}[leftmargin=*,nosep]
    \item \textbf{(Exact decomposition.)}
    $h_k = \rho_k^{(\delta)} \cdot \rho_k^{(A)} \cdot [\Lambda_\delta]_{kk} \cdot [\Lambda_A]_{kk}$.
    \item \textbf{($\alpha$ as alignment slope.)} With $\Phi_k = \rho_k^{(\delta)}\rho_k^{(A)}$ and power-law fit $h_k = c\sigma_k^\alpha$:
    \begin{equation}\label{eq:alpha-decomp}
        \alpha = \frac{d \log \Lambda_k}{d \log \sigma_k} + \frac{d \log \Phi_k}{d \log \sigma_k}, \quad \Lambda_k = [\Lambda_\delta]_{kk}[\Lambda_A]_{kk}
    \end{equation}
    \item \textbf{(Canonical form.)} When $Q_\delta = U$, $Q_A = V$ (perfect alignment), $[\Lambda_\delta]_{kk}[\Lambda_A]_{kk} = \sigma_k^2/\cos^2\theta_k$ and
    \begin{equation}\label{eq:alpha-general}
        \alpha = 2 + \frac{d \log \Phi_k}{d \log \sigma_k} - \frac{d \log \cos^2\!\theta_k}{d \log \sigma_k}
    \end{equation}
    where $\cos^2\theta_k = h_k^{\mathrm{exact}}/h_k^{\mathrm{Kron}}$ measures the GN-to-exact gap.
\end{enumerate}
\end{theorem}

\begin{proof}
\textbf{Part 1.} The curvature along $u_k v_k^\top$ is $h_k = \vecop(u_k v_k^\top)^\top (C_\delta \otimes C_A) \vecop(u_k v_k^\top) = (u_k^\top C_\delta u_k)(v_k^\top C_A v_k)$. By Definition~\ref{def:alignment}, $u_k^\top C_\delta u_k = \rho_k^{(\delta)} [\Lambda_\delta]_{kk}$ and $v_k^\top C_A v_k = \rho_k^{(A)} [\Lambda_A]_{kk}$, giving the decomposition.

\textbf{Part 2.} Taking logs: $\log h_k = \log\Phi_k + \log\Lambda_k$. If $h_k = c\sigma_k^\alpha$, then $d\log h_k / d\log\sigma_k = \alpha$. Differentiating the decomposition yields Eq.~\eqref{eq:alpha-decomp}.

\textbf{Part 3.} When $Q_\delta = U$ and $Q_A = V$, we have $\rho_k^{(\delta)} = \rho_k^{(A)} = 1$, so $\Phi_k = 1$. Under the alignment condition of Theorem~\ref{thm:curvature-baseline}, $[\Lambda_\delta]_{kk} [\Lambda_A]_{kk} = \sigma_k^2 / \cos^2\theta_k$, hence $d\log\Lambda_k / d\log\sigma_k = 2 - d\log\cos^2\theta_k / d\log\sigma_k$. Substituting into Part~2 with $d\log\Phi_k = 0$ gives Eq.~\eqref{eq:alpha-general}.
\end{proof}

\begin{remark}
Equation~\eqref{eq:alpha-general} is the central result: $\alpha$ deviates from $2$ iff $\Phi_k$ or $\cos^2\theta_k$ correlate with $\sigma_k$. When both are $k$-independent, $\alpha = 2$ regardless of their magnitudes.
\end{remark}

\begin{theorem}[Baseline: $\alpha = 2$ under perfect alignment]\label{thm:curvature-baseline}
When $C_\delta$, $C_A$ share the singular/eigen-basis of $G$ and per-sample $(u_k^\top\delta_i)^2$, $(v_k^\top a_i)^2$ are approximately independent across samples, $h_k \propto \sigma_k^2$.
\end{theorem}

\section{Architectural Mechanisms}\label{sec:mechanisms}

We now show how each architectural component produces a specific pattern in $\Phi_k$ vs.\ $\sigma_k$.

\subsection{LayerNorm}

\begin{theorem}[LayerNorm flattens activation alignment]\label{thm:layernorm}
For $y = W \cdot \mathrm{LN}(x)$, the activation covariance $C_A^{\mathrm{LN}} = \Gamma \hat{C} \Gamma$ with $\hat{C}$ the normalized-input covariance. The condition number $\kappa(\hat{C}|_{\mathbf{1}^\perp}) = 1 + O(\sqrt{n/B})$, and $\rho_k^{(A)}$ is confined to a $\sigma_k$-independent band. Thus $d\log\rho_k^{(A)}/d\log\sigma_k \approx 0$, reducing $\alpha$ from $2$ toward $1$.
\end{theorem}

\begin{corollary}\label{cor:ln-alpha}
With $\cos^2\theta_k$ approximately constant, $\alpha \approx 1 + d\log\rho_k^{(\delta)}/d\log\sigma_k$. If $u_k^\top C_\delta u_k \propto \sigma_k$, then $\alpha \approx 1$.
\end{corollary}

\subsection{Residual connections}

\begin{proposition}[Residual decorrelation]\label{prop:residual}
In $y = x + f(x)$ with skip-dominated regime ($\|J_f^\top \delta_{\mathrm{out}}\| \ll \|\delta_{\mathrm{out}}\|$), $C_\delta \approx C_{\delta_{\mathrm{out}}}$ whose eigenvectors bear no relation to current-layer $u_k$. Thus $\rho_k^{(\delta)}$ has no systematic $\sigma_k$ dependence and $d\log\rho_k^{(\delta)}/d\log\sigma_k \approx 0$.
\end{proposition}

Combined with LayerNorm, $\alpha \approx 2 - d\log\cos^2\theta_k/d\log\sigma_k$, explaining MLP up-projections with $\alpha \approx 0.8$.

\subsection{Softmax concentration}

\begin{proposition}[Softmax creates $\alpha > 2$]\label{prop:softmax}
For output layer with $c$ classes, $C_\delta$ has rank $\leq c$. Top singular vectors of $G$ align with top eigenvectors of $C_\delta$, giving $\rho_k^{(\delta)} \approx 1$ for $k \leq c$ and $\rho_k^{(\delta)} \ll 1$ beyond---a sharp drop creating $d\log\Phi_k/d\log\sigma_k > 0$ and $\alpha > 2$.
\end{proposition}

\subsection{The conv gap: GN vs.\ Kronecker}

For conv layers, $\cos^2\theta_k \sim 10^3$--$10^7$ yet $\alpha \approx 2$. The \emph{concentration ratio} $\bar{h}^{\mathrm{SoP}}/\bar{h}^{\mathrm{PoS}} \approx 1$ (Kronecker factorization is accurate), while $\cos^2\theta_k \gg 1$ (GN-to-exact gap from weight-sharing). What matters for $\alpha$ is the \emph{slope} $d\log\cos^2\theta_k/d\log\sigma_k \approx 0$, not the magnitude (Appendix~\ref{app:conv-gap}).

\subsection{Refined prediction: LayerNorm is necessary but not sufficient}

ViT-Tiny on CIFAR-10 has LayerNorm yet $\alpha \approx 2$. Controlled ablations (Appendix~\ref{app:ablations}) show that large output dimension creates sparse, anisotropic $C_\delta$, pulling $\alpha$ below $2$ on CIFAR-100; switching GPT-2 from LM to binary classification raises $\alpha$ toward $2$. The corrected condition: $\alpha < 2$ requires \textbf{both} LayerNorm \textbf{and} sufficiently large output dimension.

\section{Spectral Transfer and Intrinsic Dimension}\label{sec:spectral}

\subsection{The Alpha Triangle}

Three exponents summarize the spectral structure: curvature $\alpha$ ($h_k \propto \sigma_k^\alpha$), effective gradient rank decay $\gamma$ (slope of $\log \sigma_k$ vs.\ $\log k$ on top-$k$ directions), and Hessian decay $s$ ($\lambda_k \sim k^{-s}$ along comparable directions)~\cite{tang2025hessian}.

\begin{theorem}[Spectral transfer identity]\label{thm:triangle}
Suppose $\sigma_k \sim k^{-\gamma}$ and $h_k = c\sigma_k^\alpha$ on a rank interval. Then $\lambda_k \sim k^{-s}$ with $s = \alpha\gamma$.
\end{theorem}

\begin{proof}
$\lambda_k = h_k = c(c' k^{-\gamma})^\alpha = c'' k^{-\alpha\gamma}$.
\end{proof}

\begin{remark}[Why this is not a tautology]\label{rem:not-tautology}
The identity $s = \alpha\gamma$ is algebraic under the local power-law ansatz---the one-line proof is intentional. The empirical content is threefold: \textbf{(i)}~$\alpha$ and $\gamma$ are fit on \emph{different} data ($(\sigma_k, h_k)$ from HVPs vs.\ $(\log k, \log\sigma_k)$ from SVD), yet their product recovers $s$ from rank-ordered curvatures to median error \textbf{1.9\%} on CIFAR, \textbf{1.0\%} on Tiny-ImageNet, \textbf{1.6\%} on ImageNet-1K (Figure~\ref{fig:triangle}); \textbf{(ii)}~$\sigma_k$ is \emph{not} globally $k^{-\gamma}$---BIC selects log-normal over power-law on most layers (Figure~\ref{fig:rank-profile}; Appendix~\ref{app:gamma-profile})---yet the effective exponent $\gamma$ on the top-$k$ window still produces a tight identity; and \textbf{(iii)}~the identity holds across 93 layers spanning five architectures and three datasets with no free parameters. A tautology produces no testable predictions; this identity is violated whenever the local power-law ansatz breaks down, and it \emph{does} break at boundary layers (Appendix~\ref{app:triangle}).
\end{remark}

\begin{figure}[t]
\centering
\includegraphics[width=0.58\linewidth]{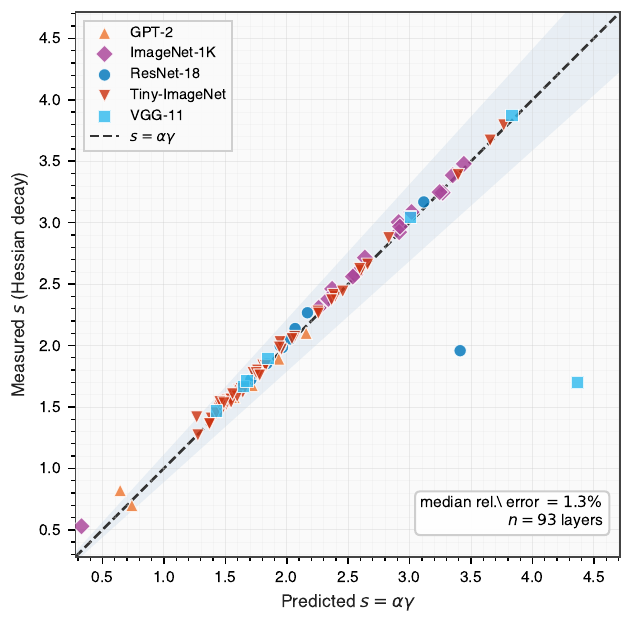}
\caption{Measured Hessian decay $s$ vs.\ predicted $s=\alpha\gamma$ across 93 layers and five architectures/datasets. Points near the diagonal validate the spectral transfer identity; shaded band is $\pm10\%$. Color indicates data source.}
\label{fig:triangle}
\end{figure}

\begin{figure}[t]
\centering
\includegraphics[width=\linewidth]{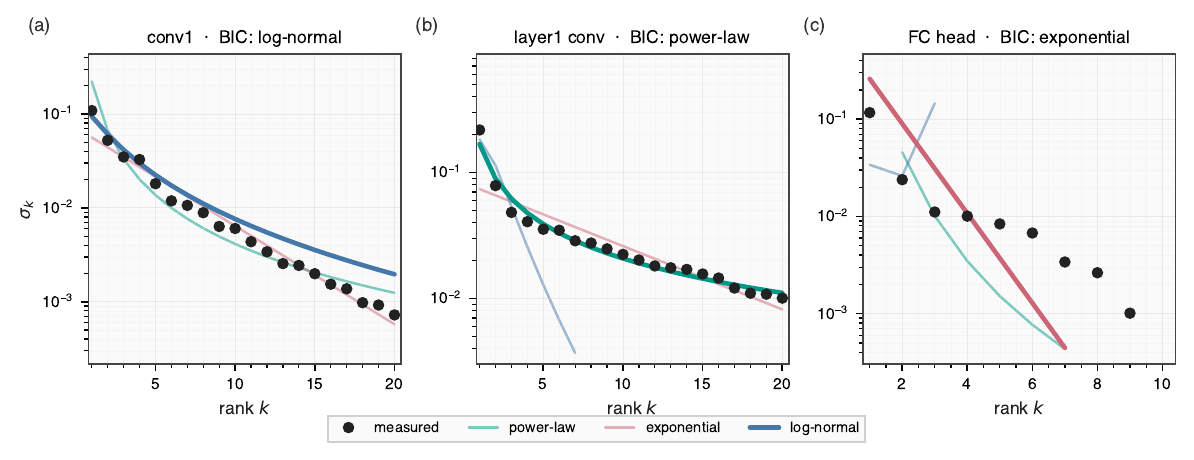}
\caption{Gradient singular value rank profiles on ResNet-18 with BIC-selected models. \textbf{(a--c)} Measured $\sigma_k$ (points) vs.\ power-law, exponential, and log-normal fits (lines); bold curve is the BIC winner. Log-normal fits are clipped to the data range on the FC head panel.}
\label{fig:rank-profile}
\end{figure}

\begin{corollary}[Connection to HT-SR]\label{cor:htsr}
Martin--Mahoney's weight tail exponent $\alpha_{\mathrm{weight}}$~\cite{martin2021implicit} is constrained by $s = \alpha\gamma$ governing the Hessian spectrum. We find $\alpha_{\mathrm{weight}} < s$ consistently, consistent with heavier weight tails correlating with better generalization.
\end{corollary}

\subsection{Near-rank-one curvature}

\begin{theorem}[Zeta function bound]\label{thm:rank-one}
With $h_k \sim k^{-\alpha\gamma}$, the participation ratio $\mathrm{PR}_h = (\sum h_k)^2 / \sum h_k^2 \leq \zeta(\alpha\gamma)^2 / \zeta(2\alpha\gamma)$, bounded for $\alpha\gamma > 1$. For $\alpha\gamma = 2$: $\mathrm{PR}_h \lesssim 2.5$.
\end{theorem}

\noindent Despite hundreds of singular directions, curvature concentrates onto $\mathrm{PR}_h \in [1.0, 1.8]$ while gradients span moderate-dimensional subspaces ($\mathrm{PR}_{\mathrm{grad}} \in [1.8, 58]$; Appendix~\ref{app:pr-table}). Width scaling from $n=64$ to $512$ changes PR by only 0.7 (Appendix~\ref{app:pr-scaling}).

\subsection{From decomposition to Spectral Newton}\label{sec:optimizer-validation}

The decomposition predicts \emph{which layers} admit Kronecker-style curvature approximations: when $|\alpha - 2|$ is small and the concentration ratio $\bar{h}^{\mathrm{SoP}}/\bar{h}^{\mathrm{PoS}} \approx 1$, the Kronecker product-of-means matches the exact mean-of-products Hessian along gradient directions (Figure~\ref{fig:kfac-pred}). Conversely, layers with $\alpha \ll 2$ (LayerNorm + large output dimension) or $\alpha \gg 2$ (softmax heads) break alignment; full Hessian or task-specific preconditioning is required, consistent with K-FAC/Shampoo using richer structure than diagonal Adam.

\paragraph{Prescriptive transfer function.}
When $h_k \propto \sigma_k^\alpha$, the optimal step along singular direction $k$ scales as $T(\sigma_k;\alpha) = \sigma_k/(\sigma_k^\alpha + d)$ rather than a uniform spectral flattening (Muon) or a single scalar per layer. For convolutional layers with $\alpha \approx 2$, this reduces to Spectral Newton (SN): $\Delta W = U\,\diag(\sigma_k/(\sigma_k^2+d))\,V^\top$ for $G = U\Sigma V^\top$. Per-layer $\alpha$ from Table~\ref{tab:arch-sweep} selects the family member---$\alpha \approx 1$ for transformers suggests $T(\sigma) = \sigma/(\sigma+d)$ instead.

\begin{figure}[t]
\centering
\includegraphics[width=\linewidth]{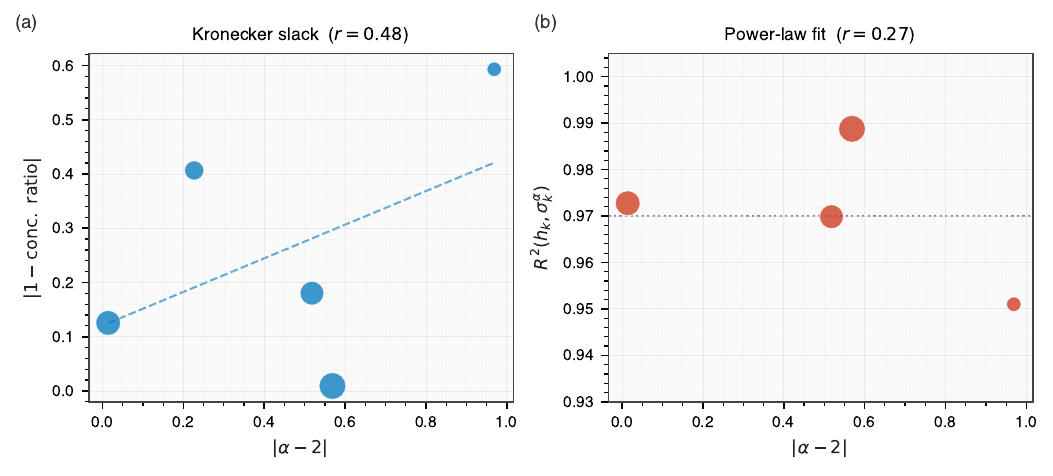}
\caption{Kronecker factorization quality vs.\ $|\alpha-2|$ on ResNet-18 conv layers (+ FC control). \textbf{(a)} Deviation of concentration ratio from 1 (marker size $\propto$ accuracy). \textbf{(b)} $R^2$ of $h_k \propto \sigma_k^\alpha$; dotted line at 0.97.}
\label{fig:kfac-pred}
\end{figure}

\paragraph{Directional evidence.}
As a proof of concept, we test SN on ResNet-18/50 vision tasks where median $\alpha \approx 2$ (Table~\ref{tab:arch-sweep}). Using matrix-function recovery after polar decomposition at $\sim$1.8$\times$ Muon cost, SN outperforms AdamW at matched epoch budgets on all three benchmarks (Table~\ref{tab:sn-vision}), with gains of $+1.1\%$ on CIFAR-10 and $+2.4$--$3.5\%$ on the harder CIFAR-100. These are single-seed runs reported as directional evidence that the decomposition yields a useful optimizer, not as a comprehensive benchmark; multi-seed evaluation and LLM-scale testing are deferred to future work. Crucially, \emph{scalar} per-layer learning-rate scaling from $s=\alpha\gamma$ does \emph{not} reproduce this improvement (Appendix~\ref{app:scalar-precond}): curvature correction is directional, not a single scale per layer.

\begin{table}[t]
\centering
\caption{Spectral Newton vs.\ baselines (test accuracy \%, matched epochs, seed 42). SN uses damping $d{=}0.1$.}
\label{tab:sn-vision}
\small
\begin{tabular}{llccc}
\toprule
\textbf{Model} & \textbf{Dataset} & Epochs & AdamW & SN \\
\midrule
ResNet-18 & CIFAR-10 & 30 & 92.7 & \textbf{93.8} \\
ResNet-18 & CIFAR-100 & 30 & 71.5 & \textbf{75.0} \\
ResNet-50 & CIFAR-100 & 30 & 72.6 & \textbf{75.0} \\
\bottomrule
\end{tabular}
\end{table}

\section{Discussion}\label{sec:discussion}

\paragraph{Spectral asymptotics interpretation.}
The identity $\alpha = 2 + d\log\Phi_k/d\log\sigma_k$ is a Weyl-law analog: the decay exponent of curvature along gradient directions encodes the \emph{alignment geometry} of the loss Hessian restricted to the gradient subspace. Just as $\lambda_k \sim k^{2/d}$ reveals manifold dimension $d$, $\alpha$ reveals how Kronecker factor eigenbases align with gradient singular directions.

\paragraph{Optimizer implications.}
The decomposition yields curvature-adaptive transfer functions $T(\sigma;\alpha) = \sigma/(\sigma^\alpha + d)$ with architecture-dependent $\alpha$. As directional evidence, Spectral Newton implements this prescription in the gradient singular basis and outperforms AdamW on vision benchmarks where $\alpha \approx 2$ (Table~\ref{tab:sn-vision}; single-seed, proof of concept). The deeper structural insight is that Kronecker-style methods are justified when $|\alpha-2|$ is small (Figure~\ref{fig:kfac-pred}), and that scalar per-layer scaling from $s=\alpha\gamma$ fails because eigenvector alignment (Theorem~\ref{thm:spectral-alignment}) is directional---the theory prescribes $T(\sigma_k;\alpha_\ell)$ per singular direction, not one learning rate per layer.

\paragraph{Gradient rank profile and phase transitions.}
BIC favors log-normal-in-rank over strict power-law on most CIFAR layers (Figure~\ref{fig:rank-profile}). During training, the winning rank model shifts: Figure~\ref{fig:phase-transition} shows VGG-16 on Tiny-ImageNet transitioning from log-normal dominance at initialization toward mixed power-law/exponential regimes by epoch 15--25 (slower than ResNet-50, where the transition peaks by epoch 5). Thus $\gamma$ is an effective exponent; Appendix~\ref{app:kronecker-identity} shows that the same Kronecker factors predict the rank-profile shape.

\begin{figure}[t]
\centering
\includegraphics[width=0.68\linewidth]{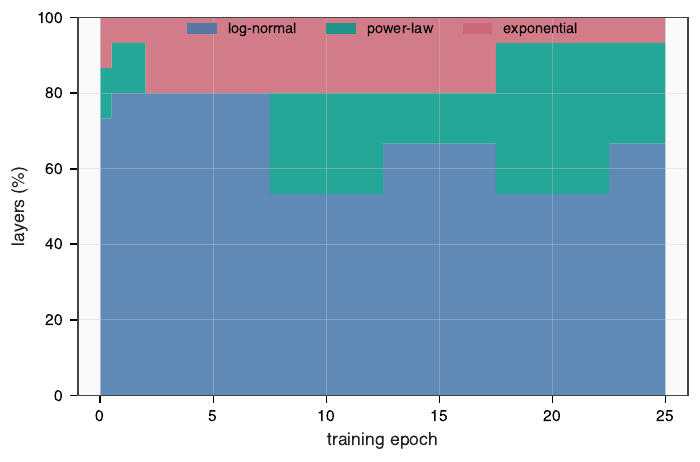}
\caption{Fraction of VGG-16 layers whose top-$k$ gradient spectrum is best fit by each BIC model vs.\ training epoch (Tiny-ImageNet). Log-normal dominance at init gives way to power-law and exponential wins---the \emph{spectral phase transition}.}
\label{fig:phase-transition}
\end{figure}

\paragraph{Limitations and open questions.}
\textbf{Scale.} Our largest models are GPT-2 (6-layer, 85M) and ResNet-50; the theory is architecture-generic and makes concrete, falsifiable predictions at larger scales---for instance, Theorem~\ref{thm:spectral-alignment} predicts that ViT-Large attention layers will exhibit $\alpha \approx 1$ (driven by LayerNorm-induced alignment collapse, Section~\ref{sec:mechanisms}) and that LLaMA-7B MLP layers will follow $\alpha < 1.2$ with the same Kronecker decomposition---but we have not yet verified these. Validation at the 1--10B parameter regime is the most important next step.
\textbf{Spectral Newton.} SN results (Table~\ref{tab:sn-vision}) are single-seed, single-architecture, and serve as directional evidence rather than a comprehensive optimizer comparison. Multi-seed evaluation with wall-clock accounting is needed.
\begin{enumerate}[leftmargin=*,nosep]
    \item Extend the zeta PR bound to log-normal or exponential rank spectra via Tauberian theorems.
    \item Validate the full Kronecker spectral theory at LLM scale (ViT-Large, LLaMA-7B, GPT-3 class).
    \item Connect spectral phase-transition timing to gradient flow on $(C_\delta(t), C_A(t))$.
\end{enumerate}

\section*{Reproducibility Statement}

All empirical claims are reproduced from scripts that measure exact Hessian-vector products per layer. The full reproducibility package is available at \url{https://github.com/9D-Labs/9d-spectral-alignment-decomposition}:
\begin{itemize}[leftmargin=*,nosep]
    \item \textbf{Frozen outputs:} 15 JSON files in \texttt{results/} covering all tables and figures.
    \item \textbf{Claim verification:} \texttt{verify\_claims.py} checks 27 quantitative claims against JSON; exits~1 on mismatch.
    \item \textbf{Figures:} \texttt{generate\_figures.py} and \texttt{analyze\_kfac\_prediction.py} build all 6 main-text plots from JSON (no GPU).
    \item \textbf{One command:} \texttt{make all} verifies claims, regenerates figures, and compiles the PDF.
    \item \textbf{Full reproduction:} \texttt{EXPERIMENTS.md} documents the measurement protocol for all 10 experiments (Modal A10G, $\sim$3--4\,h total).
\end{itemize}

\section*{Acknowledgments}
We thank the 9D Labs team for compute infrastructure and experimental support.

\appendix

\section{HVP Validation Table}\label{app:hvp-table}

\begin{center}
\small
\begin{tabular}{lccc}
\toprule
\textbf{Layer} & $R^2(h, \sigma^2)$ & $R^2(h_{\mathrm{GN}}, \sigma^2)$ & $R^2(h, h_{\mathrm{GN}})$ \\
\midrule
conv1 & 0.999 & 0.925 & 0.929 \\
layer2.0.conv1 & 0.998 & 0.998 & 0.997 \\
layer4.0.conv2 & 0.995 & 0.995 & 0.998 \\
fc (10 classes) & 0.694 & 0.869 & 0.786 \\
\midrule
\textbf{Median (21 layers)} & \textbf{0.998} & \textbf{0.992} & \textbf{0.996} \\
\bottomrule
\end{tabular}
\end{center}

\section{Participation Ratio Table}\label{app:pr-table}

\begin{center}
\small
\begin{tabular}{lrcc}
\toprule
\textbf{Layer} & $n$ & Curv.\ PR & Grad.\ PR \\
\midrule
ResNet conv1 & 27 & 1.7 & 4.6 \\
ResNet layer3.conv1 & 256 & 1.2 & 37.4 \\
GPT-2 block5.mlp.0 & 384 & 1.0 & 2.6 \\
\bottomrule
\end{tabular}
\end{center}

\section{Controlled Ablations}\label{app:ablations}

\begin{center}
\small
\begin{tabular}{lccc}
\toprule
\textbf{Layer (ViT)} & CIFAR-10 $\alpha$ & CIFAR-100 $\alpha$ & $\Delta\alpha$ \\
\midrule
Attention QKV & 2.25 & 1.74 & $-0.51$ \\
MLP up & 2.12 & 1.62 & $-0.50$ \\
\bottomrule
\end{tabular}
\end{center}

\begin{center}
\small
\begin{tabular}{lccc}
\toprule
\textbf{Layer (GPT-2)} & LM $\alpha$ & CLS $\alpha$ & $\Delta\alpha$ \\
\midrule
MLP up & 0.73 & 1.59 & $+0.86$ \\
MLP down & 1.15 & 1.85 & $+0.69$ \\
\bottomrule
\end{tabular}
\end{center}

\section{Conv Gap Table}\label{app:conv-gap}

\begin{center}
\small
\begin{tabular}{lccc}
\toprule
\textbf{Layer} & $\cos^2\theta_k$ & Conc.\ ratio & $\alpha$ \\
\midrule
conv1 & $10^7$ scale & 0.88 & 1.99 \\
layer2.0.conv1 & $10^5$ scale & 0.98 & 2.57 \\
layer4.0.conv2 & $10^3$ scale & 0.59 & 2.23 \\
fc & $10^4$ scale & 0.88 & 1.64 \\
\bottomrule
\end{tabular}
\end{center}

\section{Training Dynamics of $\alpha$}\label{app:training-dynamics}

ResNet-18 on CIFAR-10, $\alpha$ at epochs 0, 1, 3, 5, 10, 15, 20:

\begin{center}
\small
\begin{tabular}{lccccccc}
\toprule
Layer & E0 & E1 & E5 & E10 & E15 & E20 \\
\midrule
conv1 & 1.45 & 1.83 & 1.78 & 1.79 & 1.79 & 1.81 \\
layer2.conv2 & 1.85 & 1.89 & 1.87 & 1.99 & 1.94 & 2.12 \\
layer4.conv2 & 1.90 & 1.92 & 1.98 & 1.88 & 1.80 & 1.90 \\
fc & 1.99 & 1.05 & 1.33 & 1.16 & 1.31 & 1.57 \\
\bottomrule
\end{tabular}
\end{center}

Conv layers converge to $\alpha \approx 2$; FC layers decrease as softmax concentration develops.

\section{Full Spectrum Analysis}\label{app:full-spectrum}

$\alpha$ by quartile of singular index $k$ (ResNet-18, top-50 directions):

\begin{center}
\small
\begin{tabular}{lcccc}
\toprule
Layer & Top-12 & Mid-hi & Mid-lo & Bottom \\
\midrule
layer2.conv2 & 2.07 & 2.35 & 2.56 & 2.59 \\
layer3.conv2 & 2.07 & 2.07 & 2.79 & 1.93 \\
\bottomrule
\end{tabular}
\end{center}

\section{Gradient Rank Profile and $\gamma$}\label{app:gamma-profile}

We compare three rank-ordered models for $\sigma_k$: power law $\sigma_k \propto k^{-\gamma}$, exponential $\sigma_k \propto e^{-\beta k}$, and log-normal in rank ($\log\sigma_k$ quadratic in $\log k$), using BIC.

\begin{center}
\small
\begin{tabular}{lccc}
\toprule
\textbf{Dataset / $k$} & Power law & Exponential & Log-normal \\
\midrule
CIFAR triangle JSON ($k{=}20$, 24 layers) & 1 & 2 & \textbf{21} \\
Tiny-ImageNet init ($k{=}100$, 54 layers) & 12 & 3 & \textbf{39} \\
Tiny-ImageNet trained ($k{=}100$, 54 layers) & 0 & 15 & \textbf{39} \\
\bottomrule
\end{tabular}
\end{center}

\noindent Exponential BIC wins increase during training (6$\to$15 layers from epoch 0 to 25) but do not dominate at initialization. Fitted $\gamma$ is therefore an \textbf{effective} log-log slope over the measurement window; the Alpha Triangle uses this $\gamma$ consistently. See the reproducibility package for reproduction details.

\section{Full Alpha Triangle Table}\label{app:triangle}

CIFAR-10 experiments: 20 layers, median relative error 1.9\% ($n{=}19$ well-conditioned at $\leq$10\% rel.\ error). Tiny-ImageNet ResNet-50: 54 layers, median 1.0\%. CNN layers: $R^2_\gamma > 0.88$ on CIFAR; GPT-2: strong $\alpha$ fits, weaker $\gamma$ power-law $R^2$ but identity still holds.

\begin{center}
\scriptsize
\begin{tabular}{lcccccc}
\toprule
Layer & $\gamma$ & $\alpha$ & $s$ & $\alpha\gamma$ & Err & Rel.\ \\
\midrule
\multicolumn{7}{l}{\textit{ResNet-18}} \\
conv1 & 1.73 & 1.80 & 3.17 & 3.12 & 0.05 & 1.7\% \\
layer1.0.conv1 & 0.91 & 1.88 & 1.72 & 1.71 & 0.01 & 0.8\% \\
layer4.0.conv2 & 0.77 & 2.81 & 2.27 & 2.17 & 0.10 & 4.3\% \\
\multicolumn{7}{l}{\textit{VGG-11}} \\
features.0 & 1.78 & 2.16 & 3.87 & 3.84 & 0.04 & 0.9\% \\
classifier.0 & 1.64 & 1.83 & 3.04 & 3.01 & 0.04 & 1.3\% \\
\multicolumn{7}{l}{\textit{GPT-2}} \\
block2.mlp.0 & 0.76 & 2.08 & 1.58 & 1.57 & 0.01 & 0.4\% \\
block5.mlp.2 & 1.06 & 2.04 & 2.10 & 2.16 & 0.06 & 2.7\% \\
output\_head & 0.96 & 2.02 & 1.89 & 1.93 & 0.04 & 2.3\% \\
\bottomrule
\end{tabular}
\end{center}

\section{Width Scaling of Participation Ratio}\label{app:pr-scaling}

Single-hidden-layer MLPs on CIFAR-10: PR grows from 2.5 ($n=64$) to 3.3 ($n=512$) despite 8$\times$ width---bounded, not linear in $n$.

\section{Proof Details}\label{app:proofs}

\subsection{LayerNorm (Theorem~\ref{thm:layernorm})}
$C_A^{\mathrm{LN}} = \Gamma \hat{C} \Gamma$ with $\hat{x}_i^\top \mathbf{1} = 0$. By Bai--Yin~\cite{baiyin1993}, $\kappa(\hat{C}|_{\mathbf{1}^\perp}) = 1 + O(\sqrt{n/B})$. The Rayleigh quotient $\rho_k^{(A)}$ is bounded in a $k$-independent interval, so $d\log\rho_k^{(A)}/d\log\sigma_k \to 0$.

\subsection{Softmax (Proposition~\ref{prop:softmax})}
$C_\delta$ has rank $\leq c$. Top-$c$ left singular vectors of $G = \frac{1}{B}\delta^\top A$ align with eigenvectors of $C_\delta$ when predictions are concentrated, giving $\rho_k^{(\delta)} \approx 1$ for $k \leq c$ and $\ll 1$ beyond---positive $d\log\Phi_k/d\log\sigma_k$.

\subsection{Emergence during training}
$\alpha$ converges to layer-type attractors because alignment ratios evolve predictably: conv layers tighten Kronecker alignment ($\Phi_k \to 1$, $\alpha \to 2$); FC layers develop softmax anti-alignment ($\alpha \to 1$); transformer layers maintain the LayerNorm barrier ($\alpha \approx 1$ throughout).

\section{VGG-16 Spectral Trajectory}\label{app:vgg-trajectory}

We repeat the spectral trajectory analysis (Appendix~\ref{app:gamma-profile}) on VGG-16 / Tiny-ImageNet (15 conv layers, 25 epochs, top-$k{=}100$). BIC model counts by epoch:

\begin{center}
\small
\begin{tabular}{lccc}
\toprule
Epoch & Power law & Exponential & Log-normal \\
\midrule
0 & 2 & 2 & 11 \\
5 & 2 & 3 & 10 \\
10 & 3 & 3 & 9 \\
15 & 5 & 5 & 5 \\
20 & 6 & 4 & 5 \\
25 & 4 & 5 & 6 \\
\bottomrule
\end{tabular}
\end{center}

\noindent VGG-16 exhibits the same spectral phase transition as ResNet-50: log-normal dominance at initialization gives way to a mixed power-law/exponential regime after training. The transition is \emph{slower} (power-law peaks at epoch 15--20, vs.\ epoch 5 in ResNet-50), consistent with the absence of skip connections delaying spectral mixing through the backward pass. This confirms that the phase transition is architecture-general, not an artifact of residual structure.

\section{Kronecker Spectral Identity}\label{app:kronecker-identity}

The Spectral Alignment Decomposition (Theorem~\ref{thm:spectral-alignment}) determines $\alpha$; we show here that the same Kronecker factors also predict the \emph{shape} of the gradient rank profile.

\paragraph{Separable variance structure.}
In the Kronecker eigenbasis, the rotated gradient $\widetilde{G} = Q_\delta^\top G Q_A$ has entry variance
$\mathrm{Var}(\widetilde{G}_{ij}) \propto \lambda_i^{(\delta)} \lambda_j^{(A)},$
a rank-one separable profile. The singular value distribution of matrices with separable variance is characterized by the Couillet--Hachem Stieltjes transform~\cite{couillet2014separable}: the limiting spectral measure of $\widetilde{G}^\top \widetilde{G}$ depends only on the empirical distributions of $\{\lambda_i^{(\delta)}\}$ and $\{\lambda_j^{(A)}\}$.

\paragraph{Empirical validation.}
On 15 ResNet-50 layers (Tiny-ImageNet, 25 epochs), we compute Kronecker factors $C_\delta$, $C_A$ via backward hooks (im2col for conv layers) and compare their spectra to the gradient SVD:

\begin{center}
\small
\begin{tabular}{lc}
\toprule
\textbf{Metric} & \textbf{Value} \\
\midrule
$r(\log h_k^{\mathrm{Kron}},\; \log \sigma_k^2)$ & \textbf{0.974} (min 0.921, max 0.988) \\
Factor-to-gradient BIC match & \textbf{87\%} (13/15 layers) \\
$\min(r_{\mathrm{eff}})$: EXP gradient layers & 4.7 \\
$\min(r_{\mathrm{eff}})$: log-normal gradient layers & 19.4 \\
\bottomrule
\end{tabular}
\end{center}

\noindent When one Kronecker factor has low effective rank ($r_{\mathrm{eff}} < 10$), the gradient spectrum is exponential (hard spectral cutoff); when both have $r_{\mathrm{eff}} > 10$, the gradient spectrum is log-normal (CLT on the separable product). The factor BIC model matches the gradient's in 87\% of layers, with $C_A$ being the dominant driver (73\%).

\paragraph{Theoretical bridge.}
Three supporting results formalize the mechanism: (i)~a \emph{rank-truncation bound} (Weyl inequality) for exponential $\sigma_k$ when one factor has low $r_{\mathrm{eff}}$; (ii)~an \emph{effective rank bound} on $K_{\mathrm{eff}}(G)$; (iii)~a \emph{log-normal rank profile theorem}: nonzero Couillet--Hachem quantile curvature implies BIC selects log-normal over power-law when $\min(r_{\mathrm{eff}}) \gtrsim 10$ (50/50 synthetic verification). Full proofs are deferred to a forthcoming manuscript.

\section{Scalar spectral preconditioner (negative control)}\label{app:scalar-precond}

A cheaper alternative to SN is Adam with per-layer learning rate $\eta_\ell \propto (s_\ell / \bar{s})^\beta$ using measured $s_\ell = \alpha_\ell \gamma_\ell$. Across $\beta \in \{0.5, 1.0, 1.5\}$, three seeds, and controls (inverse scaling, random per-layer noise), best test accuracy on ResNet-18 / CIFAR-10 (100 epochs) differs from Adam by at most $\sim$0.2\%, within seed variance:

\begin{center}
\small
\begin{tabular}{lcc}
\toprule
\textbf{Optimizer} & Mean best (\%) & $\Delta$ vs.\ Adam \\
\midrule
Adam (baseline) & 93.85 & --- \\
Spectral scaling $\beta{=}1.0$ & 93.90 & $+0.05$ \\
Spectral scaling $\beta{=}1.5$ & 93.96 & $+0.11$ \\
\bottomrule
\end{tabular}
\end{center}

\noindent This confirms that knowing the eigenvalue profile via $s$ is insufficient without applying $T(\sigma_k;\alpha)$ along each gradient singular direction.

\end{document}